\documentclass{article} 
\usepackage{iclr2025,times}

\usepackage{amsmath,amsfonts,bm}

\def\eqref#1{equation~\ref{#1}}

\def\1{\bm{1}}

\DeclareMathAlphabet{\mathsfit}{\encodingdefault}{\sfdefault}{m}{sl}
\SetMathAlphabet{\mathsfit}{bold}{\encodingdefault}{\sfdefault}{bx}{n}

\usepackage{hyperref}
\usepackage{url}

\usepackage{algorithm}
\usepackage{algorithmic}
\usepackage{amsfonts}
\usepackage{amsmath}
\usepackage{amsthm}
\usepackage{amssymb}
\usepackage{booktabs}
\usepackage{graphicx}
\usepackage{makecell}

\newcommand{\Sc}{\mathcal{S}}
\newcommand{\Ac}{\mathcal{A}}
\newcommand{\Tc}{\mathcal{T}}

\newcommand{\Dc}{\mathcal{D}}

\newcommand{\Nc}{\mathcal{N}}
\newcommand{\Rc}{\mathcal{R}}
\newcommand{\Uc}{\mathcal{U}}

\newcommand{\Bs}{\boldsymbol{s}}
\newcommand{\Ba}{\boldsymbol{a}}
\newcommand{\Btau}{\boldsymbol{\tau}}
\newcommand{\hBtau}{\widehat{\boldsymbol{\tau}}}
\newcommand{\hmu}{\widehat{\mu}}
\newcommand{\Bmu}{\boldsymbol{\mu}}
\newcommand{\0}{\mathbf{0}}
\newcommand{\bI}{\mathbf{I}}
\newcommand{\bg}{\mathbf{g}}
\newcommand{\Bepsilon}{\boldsymbol{\epsilon}}
\newtheorem{prop}{Proposition}
\DeclareMathOperator{\env}{ENV}

\title{Diffusion Modulation via Environment Mechanism Modeling for Planning}

\author{Hanping Zhang, Yuhong Guo\\
School of Computer Science, Carleton University, Ottawa, Canada\\
\texttt{\{jagzhang@cmail.,yuhong.guo\}@carleton.ca} \\
}

\iclrfinalcopy 
\begin{document}

\maketitle

\begin{abstract}
Diffusion models have shown promising capabilities in trajectory generation for planning in offline 
reinforcement learning (RL). 
However, conventional diffusion-based planning methods often fail to account for the fact 
that generating trajectories in RL requires unique consistency between transitions to ensure coherence in real environments. 
This oversight can result in considerable discrepancies between the generated trajectories and the 
underlying mechanisms of a real environment. 
To address this problem, we propose a novel diffusion-based planning method, 
termed as Diffusion Modulation via Environment Mechanism Modeling (DMEMM).
DMEMM modulates diffusion model training by incorporating key RL environment mechanisms, 
particularly transition dynamics and reward functions. 
Experimental results demonstrate that DMEMM achieves state-of-the-art performance 
for planning with offline reinforcement learning.
\end{abstract}

\section{Introduction}
Offline reinforcement learning (RL) has garnered significant attention for its potential to leverage pre-collected datasets to learn effective policies without requiring further interaction with the environment \citep{levine2020offlinereinforcementlearningtutorial}. One emerging approach within this domain is the use of diffusion models for trajectory generation \citep{janner2022diffuser}. Diffusion models \citep{sohl2015deep, ho2020denoising}, initially popularized for tasks such as image synthesis, have demonstrated promising capabilities in generating coherent and diverse trajectories for planning in offline RL settings \citep{janner2022diffuser,ni2023metadiffuser,li2023efficient,goyal2023robust}. 
Nevertheless, the essential differences between mechanisms in image synthesis and RL necessitate 
specific considerations for the effective application of diffusion models in RL.

In image synthesis \citep{ho2020denoising}, diffusion models primarily aim to produce visually coherent outputs consistent in style and structure, while RL tasks demand 
environment and task oriented
consistency between transitions in the generated trajectories \citep{janner2022diffuser} to ensure that the generated sequences are not only plausible but also effective for policy learning \citep{kumar2020conservative}. 
This consistency is essential for ensuring that the sequence of actions within the generated trajectories can successfully guide the RL agent from the current state to the target state. However, conventional diffusion-based planning methods often overlook this need for transition coherence \citep{janner2022diffuser}. 
By simply adopting traditional diffusion models like DDPM, which utilize a fixed isotropic variance for Gaussian distributions, 
such diffusion-based planning models may fail to adequately capture the transition dynamics necessary for effective RL, 
leading to inaccurate trajectories and suboptimal learned policies \citep{wu2019behavior}.

To address this problem, we introduce a novel diffusion-based planning method 
called Diffusion Modulation via Environment Mechanism Modeling (DMEMM). 
This method modulates the diffusion process by integrating RL-specific environment mechanisms, 
particularly transition dynamics and reward functions, directly into the diffusion model training process on offline data,
thereby enhancing the diffusion model to better capture the underlying transition and reward structures of the offline data. 
Specifically, 
we modify the diffusion loss by weighting it with the cumulative reward, 
which biases the diffusion model towards high-reward trajectories,
and introduce two auxiliary modulation losses based on empirical transition and reward models
to regularize the trajectory diffusion process,
ensuring that the generated trajectories 
are not only plausible but also reward-optimized.
Additionally, we also utilize the transition and reward models to guide the sampling process 
during planning trajectory generation from the learned diffusion model, 
further aligning the outputs with the desired transition dynamics and reward structures. 
We conducted experiments on multiple RL environments. 
Experimental results indicate that our proposed method achieves state-of-the-art performance 
compared to previous diffusion-based planning approaches. 

This work presents a significant step forward in the application of diffusion models for trajectory generation in offline RL. 
The main contributions can be summarized as follows:
\begin{itemize}
\item We identify a critical problem in conventional diffusion model training for offline RL planning, 
where the use of fixed isotropic variance and 
the disregard for rewards 
may lead to a mismatch between generated trajectories and those desirable for RL. 
To address this issue, we propose a novel method called Diffusion Modulation via Environment Mechanism Modeling (DMEMM).
\item We incorporate RL-specific environment mechanisms, including transition dynamics and reward functions, 
into diffusion model training through loss modulation,  
enhancing the quality and consistency of the generated trajectories in a principled manner
and providing a fundamental framework for adapting diffusion models to offline RL tasks. 
\item Our experimental results demonstrate that the proposed method achieves state-of-the-art results in planning with offline RL, validating the effectiveness of our approach.
\end{itemize}

\section{Related Works}

\subsection{Offline Reinforcement Learning}
Offline reinforcement learning (RL) has gained significant traction in recent years, with various approaches proposed to address the challenges of learning from static datasets without online environment interactions.
\citet{fujimoto2019off} introduced Batch Constrained Q-Learning (BCQ) that learns a perturbation model to constrain the policy to stay close to the data distribution, mitigating the distributional shift issue.  
\citet{wu2019behavior} conducted Behavior Regularized Offline Reinforcement Learning (BRAC) that incorporates behavior regularization into actor-critic methods to prevent the policy from deviating too far from the data distribution.
Conservative Q-Learning (CQL) by \citet{kumar2020conservative} uses a conservative Q-function to underestimate out-of-distribution actions, preventing the policy from exploring unseen state-action regions.
\citet{kostrikov2021offline} conducted Implicit Q-Learning (IQL) to directly optimize the policy to match the expected Q-values under the data distribution.
\citet{goyal2023robust} introduce Robust MDPs to formulate offline RL as a robust optimization problem over the uncertainty in the dynamics model.
Planning has emerged as a powerful tool for solving offline RL tasks. 
MOReL by \citet{kidambi2020morel} was the first to integrate planning into offline RL, using a learned dynamics model to simulate trajectories and enforce conservative constraints to avoid out-of-distribution actions. 
MOPO by \citet{yu2020mopomodelbasedofflinepolicy} enhances this with uncertainty-aware planning, penalizing simulated trajectories that deviate from the offline data. 
\citet{janner2021offlinereinforcementlearningbig} proposed Offline Model Predictive Control (MPC), which uses short-horizon planning by constructing future trajectories from offline data and selecting actions. 

\subsection{Diffusion Model in Reinforcement Learning}
Diffusion models have emerged as a powerful tool for RL tasks, particularly in the areas of planning and policy optimization. \citet{janner2022planning} first introduced the idea of using diffusion models for trajectory optimization on planning in offline RL, casting it as a probabilistic model that iteratively refines trajectories. Subsequent works by \citet{li2023efficient} introduce a Latent Diffuser that generates actions in the latent space by incorporating a Score-based Diffusion Model (SDM) \citep{song2021scorebased,nichol2021improveddenoisingdiffusionprobabilistic,ho2022classifierfreediffusionguidance} and utilizes energy-based sampling to improve the overall performance of diffusion-based planning. \citet{chen2024simple} propose a Hierarchical Diffuser, which achieves hierarchical planning by breaking down planning trajectories into segments and treating intermediate states as subgoals to ensure more precise planning. More recently, \citet{ni2023metadiffuser} proposed a task-oriented conditioned diffusion planner (MetaDiffuser) for offline meta-reinforcement learning. MetaDiffuser learns a context-conditioned diffusion model that can generate task-oriented trajectories for planning across diverse tasks, demonstrating the outstanding conditional generation ability of diffusion architectures. These works highlight the versatility of diffusion models in addressing RL challenges. 

\section{Preliminaries}

Reinforcement learning (RL)~\citep{sutton2018reinforcement} can be modeled as a Markov Decision Process (MDP) $M = (\Sc, \Ac, \Tc, \Rc)$ in a given environment, where $\Sc$ denotes the state space, $\Ac$ corresponds to the action space, $\Tc: \Sc \times \Ac \to \Sc$ defines the transition dynamics, and $\Rc: \Sc \times \Ac \to \mathbb{R}$ represents the reward function.
Offline RL aims to train an RL agent from an offline dataset $\Dc$, 
consisting of a collection of trajectories $\{\Btau_1,\Btau_2,\cdots,\Btau_i,\cdots \}$,
with each trajectory $\Btau_i = (s_0^i, a_0^i, r_0^i, s_1^i, a_1^i, r_1^i, \dots, s_T^i, a_T^i, r_T^i)$
sampled from the underlying MDP in the given environment.
In particular, the task of planning in offline RL aims to generate planning trajectories from an initial state $s_0$
by simulating action sequences $\Ba_{0:T}$ and predicting future states $\Bs_{0:T}$ based on those actions. 
The objective is to learn an optimal plan function such that the cumulative reward can be maximized 
when executing the plan under the underlying MDP of the given environment.

\subsection{Planning with Diffusion Model}
\label{sec:preliminary}

Diffusion probabilistic models, commonly known as ``diffusion models'' \citep{sohl2015deep, ho2020denoising}, are a class of generative models that utilize a unique Markov chain framework. When applied to planning in offline RL, the objective is to generate best planning trajectories $\{\Btau\}$ 
by learning a diffusion model on the offline RL dataset $\Dc$.

\paragraph{Trajectory Representation}
In the diffusion model applied to RL planning, it is necessary to predict both states and actions. 
Therefore, the trajectory representation in the model is 
in an image-like matrix format. 
In particular, trajectories are represented as two-dimensional arrays~\citep{janner2022diffuser}, 
where each column corresponds to 
a state-action pair $(s_t, a_t)$ of the trajectory:
$$
\Btau = \begin{bmatrix}
s_0 & s_1 & \cdots & s_T \\
a_0 & a_1 & \cdots & a_T 
\end{bmatrix}
$$

\paragraph{Trajectory Diffusion}
The diffusion model~\citep{ho2020denoising} 
comprises two primary processes: the forward process and the reverse process. 
The forward process (diffusion process) 
is a Markov chain characterized by $q(\Btau^k | \Btau^{k-1})$
that gradually adds Gaussian noise at each time step $k\in\{1,\cdots,K\}$, 
starting from an initial clean trajectory sample $\Btau^0\sim\Dc$. 
The conditional probability is particularly defined as a Gaussian probability density function, such as:
\begin{equation} 
q(\Btau^k | \Btau^{k-1}) := \Nc(\Btau^k; (1 - \beta_k)\Btau^{k-1}, \beta_k\mathbf{I}), 
\end{equation}
with $\{{\beta}_1,\cdots, \beta_K\}$ representing a predefined variance schedule. 
By introducing $\alpha_k := 1 - \beta_k$ and $\bar{\alpha}_k := \prod_{i=1}^{k} \alpha_i$, 
one can succinctly express the diffused sample at any time step $k$ as:
\begin{equation}
\label{eqa:tau}
\Btau^k = \sqrt{\bar{\alpha}_k}\Btau^0 + \sqrt{1 - \bar{\alpha}_k}\Bepsilon, 
\end{equation}
where 
$\Bepsilon \sim \Nc(\0, \mathbf{I})$.
The reverse diffusion process is an iterative denosing procedure,
and can be modeled as a parametric Markov chain characterized by $p_\theta(\Btau^{k-1} | \Btau^k)$, 
starting from a Gaussian noise prior $\Btau^K \sim \Nc(\0, \mathbf{I})$,
such that:
\begin{align}
	p_\theta(\Btau^{k-1} | \Btau^k) &= \Nc(\Btau^{k-1}; \mu_\theta(\Btau^k, k), \sigma^2_k \mathbf{I}), 
	\label{eq:reversediff}
	\\
	\mbox{with}\;\,	\mu_\theta(\Btau^k, k) &= \frac{1}{\sqrt{\alpha_k}} \left( \Btau^k - \frac{1 - \alpha_k}{\sqrt{1 - \bar{\alpha}_k}} \epsilon_\theta(\Btau^k, k) \right).
\label{eqa:mu}
\end{align}
\paragraph{Training}
In the literature, 
the diffusion model is trained by predicting the additive noise $\epsilon$ \citep{ho2020denoising}
using the noise network 
$\epsilon_\theta(\Btau^k, k)=\epsilon_\theta( \sqrt{\bar{\alpha}_k}\Btau^0 + \sqrt{1 - \bar{\alpha}_k}\Bepsilon,k)$. 
The training loss is expressed as the mean squared error 
between the additive noise $\epsilon$ and the predicted noise $\epsilon_{\theta}(\Btau^k, k)$:
\begin{equation}
L_{\text{diff}} = \mathbb{E}_{k \sim \Uc(1,K), \Bepsilon \sim \Nc(\0, \bI), \Btau^0 \sim \Dc} 
	\left\| \Bepsilon - \epsilon_\theta( \sqrt{\bar{\alpha}_k}\Btau^0 + \sqrt{1 - \bar{\alpha}_k}\Bepsilon,k)\right\|^2
	\label{eq:Ldiff}
\end{equation}
where $\Uc(1,K)$ denotes a uniform distribution over numbers in $[1, 2, \cdots, K]$.
With the trained noise network, the diffusion model can be used to generate RL trajectories for planning 
through the reverse diffusion process characterized by Eq.(\ref{eq:reversediff}).

\section{Method}

In this section, we present our proposed diffusion approach, Diffusion Modulation via Environment Mechanism Modeling (DMEMM),
for planning in offline RL. 
This method integrates the essential transition and reward mechanisms of reinforcement learning into
an innovative modulation-based diffusion learning framework,
while maintaining isotropic covariance matrices for the diffusion Gaussian distributions
to preserve the benefits of this conventional setup%
---simplifying model complexity, stabilizing training and enhancing performance. 
Additionally, the transition and reward mechanisms are further leveraged to guide the planning phase under the trained diffusion model, 
aiming to generate optimal planning trajectories that align with both the underlying MDP of the environment and the objectives of RL. 

\subsection{Modulation of Diffusion Training}
In an RL environment, the transition dynamics and reward function are two fundamental components of the underlying MDP.
Directly applying conventional diffusion models to offline RL can lead to a mismatch 
between the generated trajectories and those optimal for the underlying MDP in RL. 
This is due to the use of isotropic covariance and the disregard for rewards in traditional diffusion models. 
To tackle this problem, we propose to modulate the diffusion model training 
by deploying a reward-aware diffusion loss and 
enforcing auxiliary regularizations on the generated trajectories based on environment transition and reward mechanisms.  

Given the offline data $\Dc$ collected from the RL environment, 
we first learn a probabilistic transition model 
$\widehat{\Tc}(s_t, a_t)$ and a reward function $\widehat{\Rc}(s_t, a_t)$ from $\Dc$ 
as regression functions to predict the next state $s_{t+1}$ and the corresponding reward $r_t$ respectively. 
These models can serve
as estimations of the underlying MDP mechanisms. 
In order to regularize diffusion model training for generating desirable trajectories, 
using the learned transition model and reward function, 
we need to express the output trajectories of the reverse diffusion process
in terms of the diffusion model parameters, $\theta$. 
To this end, we present the following proposition.

\begin{prop}
Given the reverse process encoded by Eq.(\ref{eq:reversediff}) and Eq.(\ref{eqa:mu}) in the diffusion model,
the output trajectory $\hBtau^0$ denoised from an intermediate trajectory $\Btau^k$ at step $k$ 
has the following Gaussian distribution:
\begin{align}
	&\hBtau^0\sim \Nc(\hmu_\theta(\Btau^k, k), \widehat{\sigma}^2\mathbf{I}), 
\\ 
	\mbox{where}\quad&
\hmu_\theta(\Btau^k, k) =\frac{1}{\sqrt{\bar{\alpha}_k}}\Btau^k 
- \sum_{i=1}^k \frac{1 - \alpha_i}{\sqrt{(1 - \bar{\alpha}_i) \bar{\alpha}_i}} \epsilon_\theta (\Btau^i, i).
\end{align}
\label{prop1}
\end{prop}
%
Conveniently, we can use the mean of the Gaussian distribution above directly 
as the most likely output trajectory, denoted as $\hBtau^0=\hmu_\theta(\Btau^k, k)$.
This allows us to express the denoised output trajectory explicitly 
in terms of the parametric noise network $\epsilon_\theta$, and thus the parameters $\theta$ of the diffusion model. 
Moreover, by deploying Eq.(\ref{eqa:tau}), 
we can get rid of the latent $\{\Btau^1,\cdots,\Btau^k\}$ and re-express  $\hBtau^0$
as the following function of a sampled clean trajectory $\Btau^0$ and some random noise $\Bepsilon$: 
\begin{align}
\hBtau^0_\theta(\Btau^0, k, \Bepsilon)
&= \Btau^0 + \sqrt{\frac{1 - \bar{\alpha}_k}{\bar{\alpha}_k}} \Bepsilon 
- \sum_{i=1}^k \frac{1 - \alpha_i}{\sqrt{(1 - \bar{\alpha}_i) \bar{\alpha}_i}} 
\epsilon_\theta \left( \sqrt{\bar{\alpha}_i} \Btau^0 + \sqrt{1 - \bar{\alpha}_i} \Bepsilon, i \right).
\end{align}
Next, we leverage this output trajectory function to modulate diffusion model training 
by developing modulation losses. 

\subsubsection{Transition-based Diffusion Modulation}
As previously discussed, the deployment of a fixed isotropic variance in conventional diffusion models
has the potential drawback of overlooking the underlying transition mechanisms of the RL environment. 
As a result, there can be 
potential mismatches between the transitions of generated trajectories and the underlying transition dynamics.
Consequently, the RL agent may diverge from the expected states 
when executing the planning actions generated by the diffusion model,
leading to poor planning performance. 
To address this problem, 
the first auxiliary modulation loss is designed to minimize the discrepancy between the transitions 
in the generated trajectories from the diffusion model 
and those predicted by the learned transition model $\widehat{\Tc}$, which encodes the underlying transition mechanism. 
Specifically, for each transition 
$(s_t, a_t, s_{t+1})$ 
in a generated trajectory $\hBtau^0_\theta(\Btau^0, k, \Bepsilon)$,
we minimize the mean squared error between $s_{t+1}$ and the predicted next state using the transition model $\widehat{\Tc}$.
This leads to the following transition-based diffusion modulation loss: 
\begin{align}
	L_{\text{tr}} &= \mathbb{E}_{k \sim \Uc(1,K), \Bepsilon \sim \Nc(\0, \bI), \Btau^0 \sim \Dc} 
	\Bigg[\sum_{(s_t, a_t, s_{t+1})\in \hBtau_\theta^0(\Btau^0, k, \Bepsilon)} 
	\big\|s_{t+1} - \widehat{\Tc}(s_t, a_t) \big\|^2 \Bigg]
\end{align}
Here, the expectation is taken over the uniform sampling of time step $k$ from $[1:K]$,
the random sampling of noise $\Bepsilon$ from a standard Gaussian distribution,
and the random sampling of input trajectories from the offline training data $\Dc$. 
Through function $\hBtau^0_\theta$, this loss $L_{\text{tr}}$ is a function of the diffusion model parameters $\theta$. 
By minimizing this transition-based modulation loss, 
we enforce that the generated trajectories from the diffusion model 
are consistent with the transition dynamics expressed in the offline dataset. 
This approach enhances the fidelity of the generated trajectories and improves the overall performance of the diffusion model 
in offline reinforcement learning tasks.

\subsubsection{Reward-based Diffusion Modulation}
The goal of planning is to generate trajectories that maximize cumulative rewards 
when executed under the underlying MDP of the given environment. 
Thus, focusing solely on the fit of the planning trajectories to the transition dynamics is insufficient. 
It is crucial to guide the diffusion model training to directly align with the planning objective.
Therefore, the second auxiliary modulation loss is designed to 
maximize the reward induced in the generated trajectories. 
As the trajectories generated from diffusion models do not have reward signals, 
we predict the reward scores of the state-action pairs $\{(s_t, a_t)\}$ in each 
trajectory generated through function $\hBtau^0_\theta(\cdot,\cdot,\cdot)$
using the learned reward function $\widehat{\Rc}(\cdot,\cdot)$.
Specifically, we formulate the reward-based diffusion modulation loss function 
as the following negative expected trajectory-wise cumulative reward 
from the generated trajectories: 
\begin{equation}
	L_{\text{rd}} = -\mathbb{E}_{k \sim \Uc(1,K), \Bepsilon \sim \Nc(\0, \bI), \Btau^0 \sim \Dc} 
	\Bigg[ \sum_{(s_t, a_t) \in \hBtau^0_\theta(\Btau^0, k,\Bepsilon)} \widehat{\Rc}(s_t, a_t) \Bigg]
\end{equation}
Through function $\hBtau^0_\theta$, this loss $L_{\text{rd}}$ again is a function of the diffusion model parameters $\theta$. 
By computing the expected loss over different time steps $k\in[1:K]$, different random noise $\Bepsilon$,
and all input trajectories from the offline dataset $\Dc$, 
we ensure that the modulation is consistently enforced across all instances of diffusion model training. 

By minimizing this reward-based loss, we ensure that the generated trajectories 
are not only plausible but also reward-optimized to align with the reward structure inherent in the offline data. 
This approach improves the quality of the trajectories generated from the diffusion model and enhances the overall 
policy learning process in offline reinforcement learning tasks.

\subsubsection{Reward-Aware Diffusion Loss}
In addition to the auxiliary modulation losses, we propose to further align diffusion model training with 
the goal of RL planning by devising a novel reward-aware diffusion loss to replace the original one. 
The original diffusion loss (shown in Eq.(\ref{eq:Ldiff})) minimizes the expected 
per-trajectory mean squared error
between the true additive noise and the predicted noise,
which gives equal weights to different training trajectories 
without differentiation. 
In contrast, we propose to weight each trajectory instance $\Btau^0$ from the offline dataset $\Dc$
using its normalized cumulative reward,
so that the diffusion training can focus more on the more informative trajectory instances with larger cumulative rewards. 
Specifically, we weight each training trajectory $\Btau^0$ using its normalized 
cumulative reward and formulate the following reward-aware diffusion loss: 
\begin{equation}
	L_{\text{wdiff}} = \mathbb{E}_{k \sim \Uc(1,K), \Bepsilon \sim \Nc(\0, \bI), \Btau^0 \sim \Dc} 
	\Bigg[ 
	\Bigg( \sum_{(s_t, a_t) \in \Btau^0} \frac{\Rc(s_t, a_t)}{ T_{\text{max}}\cdot r_\text{max}} \Bigg) 
	\left\| \Bepsilon - \epsilon_\theta( \sqrt{\bar{\alpha}_k}\Btau^0 + \sqrt{1 - \bar{\alpha}_k}\Bepsilon,k)\right\|^2
	\Bigg]	
\end{equation}
Here, $\sum_{(s_t, a_t) \in \Btau^0} \Rc(s_t, a_t)$ is the trajectory-wise cumulative reward 
on the original offline data instance $\Btau^0\in\Dc$; 
$T_\text{max}$ denotes the largest trajectory length 
and $r_{\text{max}}$ denotes the maximum possible per-step reward. 
By using $T_{\text{max}}\cdot r_\text{max}$ as the normalizer,
we scale the cumulative reward to a ratio within $(0,1]$ to weight the corresponding per-trajectory diffusion loss. 
This weighting mechanism biases the diffusion model toward high-reward trajectories,
ensuring that those trajectories yielding higher cumulative rewards are more accurately represented,
thus aligning diffusion training with the planning objectives in offline RL.  
This approach improves the model's performance on rare but valuable trajectories, which are crucial for effective 
reinforcement learning. 

\subsubsection{Full Modulation Framework}
The proposed full modulated diffusion model comprises all of the three loss components presented above: 
the reward-aware diffusion loss $L_{\text{wdiff}}$, 
the transition-based auxiliary modulation loss $L_{\text{tr}}$, and 
the reward-based auxiliary modulation loss $L_{\text{rd}}$. 
By integrating these loss terms together, we have the following total loss for modulated diffusion training:
\begin{equation}
\label{eqa:total_loss}
L_{\text{total}} = L_{\text{wdiff}} + \lambda_{\text{tr}} L_{\text{tr}} + \lambda_{\text{rd}} L_{\text{rd}},
\end{equation}
where $\lambda_{\text{tr}}$ and $\lambda_{\text{rd}}$ are trade-off parameters that balance the 
contributions of the transition-based and reward-based auxiliary losses, respectively. 
Standard diffusion training algorithm can be utilized to train the model $\theta$ by minimizing this total loss function. 
By employing this integrated loss function, 
we establish a comprehensive modulation framework that incorporates essential domain and task knowledge 
into diffusion model training, 
offering a general capacity of
enhancing the adaptation and broadening the applicability of diffusion models. 

\subsection{Planning with Dual Guidance}
Once trained, the diffusion model can be used to generate trajectories 
for planning during an RL agent's online interactions with the environment.
The generation procedure starts from an initial noise trajectory
$\Btau^K\sim\Nc(\0,\mathbf{I})$, 
and gradually denoises it by following the reverse diffusion process
$\Btau^{k-1}\sim\Nc(\Bmu^{k-1},\sigma_k^2\mathbf{I})$ 
for each time step $k\in \{K, K-1, \dots, 1\}$,
where $\Bmu^{k-1}$ 
is estimated through Eq. (\ref{eqa:mu}).
In each diffusion time step $k$, the first state $s_0$ of the trajectory $\Btau^k$ 
is fixed to the current state $s$ of the RL agent in the online environment 
to ensure the plan starts from it. 
The denoised trajectory $\Btau^0$ after $K$ diffusion time steps is 
treated as the plan for the RL agent, 
which is intended to maximize the RL agent's long-term performance
without extra interaction with the environment.

To further enhance the objective of planning, some previous work~\citep{janner2022diffuser} 
has utilized the learned reward function to guide the sampling process of planning. 
In this work, we propose to deploy dual guidance 
for each reverse diffusion step $k$ by exploiting both the reward function $\widehat{\Rc}$
and the transition model $\widehat{\Tc}$ learned from the offline dataset $\Dc$. 
Following previous works on conditional reverse diffusion~\citep{dhariwal2021diffusionmodelsbeatgans},
we incorporate the dual guidance by perturbing the mean of the Gaussian distribution 
$\Nc(\Bmu^{k-1},\sigma_k^2\mathbf{I})$ used for reverse diffusion sampling.
Specifically, we integrate the gradient $\bg$ of the linear combination of the reward function and transition function 
w.r.t the trajectory into $\Bmu^{k-1}$, such that 
$\Btau^{k-1}\sim\Nc(\Bmu^{k-1}+\alpha\sigma_k^2\bI\bg,\sigma_k^2\mathbf{I})$ and $\bg$ is computed as: 
\begin{equation}
\label{eqa:dualguidance}
	\bg = \sum_{t=0}^{T} \nabla_{(s_t, a_t)} \widehat{\Rc}(s_t, a_t)+\lambda 
	\sum_{t=0}^{T-1} \nabla_{(s_t, a_t)} \log \widehat{\Tc}(s_{t+1}|s_t,a_t)
\end{equation}
where $\alpha$ is a tradeoff parameter that controls the degree of guidance. 
By incorporating both the reward and transition guidance, 
we aim to enhance the planning process to generate high-quality trajectories 
that are both reward-optimized and transition-consistent,
improving the overall planning performance. 
The details of the proposed planning procedure is summarized in 
Algorithm \ref{alg:planning}.  

\begin{algorithm}[t]
\caption{Planning with Dual Guidance}
\label{alg:planning}
\begin{algorithmic}
\REQUIRE Noise network $\epsilon_\theta$, tradeoff parameter $\alpha$,
environment $\env$, covariances $\{\sigma_k^2\}$.
\STATE Initialize environment step $t=0$.
\WHILE{not finished}
\STATE Initialize noise trajectory $\Btau_t^K$: $\Btau_t^K\sim \Nc(0,\mathbf{I})$.
\FOR{diffusion step $k=K,\dots,1$}
    \STATE Compute the mean $\mu^{k-1}$ using Eq. (\ref{eqa:mu}).
    \STATE Compute the guidance $\bg$ using Eq. (\ref{eqa:dualguidance}).
    \STATE Sample next trajectory $\Btau_t^{k-1}$: $\Btau^{k-1}\sim\Nc(\Bmu^{k-1}+\alpha\sigma_k^2\bI\bg,\sigma_k^2\mathbf{I})$.
    \STATE Fix the current state $s_t$ to the trajectory: $\Btau_{t}^{k-1}(s_0)=s_t$.
\ENDFOR
\STATE Execute first action of plan $\Btau_t^0(a_0)$: $s_{t+1}=\env(s_t,\Btau_t^0(a_0))$
\STATE Increment environment step by 1: $t=t+1$
\ENDWHILE
\end{algorithmic}
\end{algorithm}
%

\section{Experiment}
In this section, we present the experimental setup and results for evaluating our proposed method, DMEMM, across various offline RL tasks. We conduct experiments on the D4RL locomotion suite and Maze2D environments to assess the performance of DMEMM compared to several state-of-the-art methods. The experiments are designed to demonstrate the effectiveness of our approach across different tasks, expert levels, and complex navigation scenarios.

\paragraph{Environments}
We conduct our experiments on D4RL \citep{d4rl} tasks to evaluate the performance of planning in offline RL settings. Initially, we focus on the D4RL locomotion suite to assess the general performance of our planning methods across different tasks and expert levels of demonstrations. The RL agents are tested on three different tasks: HalfCheetah, Hopper, and Walker2d, and three different levels of expert demonstrations: Med-Expert, Medium, and Med-Replay. We use the normalized scores provided in the D4RL \citep{d4rl} benchmarks to evaluate performance. Subsequently, we conduct experiments on Maze2D \citep{d4rl} environments to evaluate performance on maze navigation tasks.

\paragraph{Comparison Methods}
We benchmark our methods against several leading approaches in each task domain, including model-free BCQ \citep{fujimoto2019off}, BEAR \citep{kumar2019stabilizing}, CQL \citep{kumar2020conservative}, IQL \citep{kostrikov2022offline}, Decision Transformer (DT) \citep{chen2021decision}, model-based MoReL \citep{kidambi2020morel}, Trajectory Transformer (TT) \citep{janner2021trajectory}, and Reinforcement Learning via Supervised Learning (RvS) \citep{emmons2022rvs}. We also compare our methods with the standard diffusion planning method Diffuser \citep{janner2022diffuser} and a hierarchical improvement of Diffuser, PDFD \citep{pdfd2022}.

\paragraph{Implementation Details}
We adopt the main implementations of the diffusion model and reward model from \citep{janner2022diffuser}, and use an ensemble of Gaussian models as the backend for the transition model. We use a planning horizon $T$ of 100 for all locomotion tasks, 128 for block stacking, 128 for Maze2D / Multi2D U-Maze, 265 for Maze2D / Multi2D Medium, and 384 for Maze2D / Multi2D Large. We use $N = 100$ diffusion steps. Additionally, we employ a guide scale of $\alpha = 0.001$. For the tradeoff parameters, we use $\lambda_{\text{rd}} = 0.05$ for reward loss and $\lambda_{\text{td}} = 0.1$ for transition loss.

\subsection{Experimental Results on D4RL}
The experimental results summarized in Table \ref{tab:d4rl} highlight the performance of various comparison methods across different Gym tasks, with scores averaged over 5 seeds. Our proposed method, DMEMM, consistently outperforms other methods across all tasks. Notably, in the HalfCheetah environments, DMEMM achieves a 2.1-point improvement on the Med-Expert dataset, a 2.5-point increase on the Medium dataset, and an 8.0-point improvement on the Med-Replay dataset compared to the previous best results. Additionally, DMEMM shows a 5.9-point increase on the Med-Replay Hopper task, demonstrating that DMEMM effectively extracts valuable information, particularly from data that is not purely expert-level.

In most tasks, DMEMM outperforms HD-DA, another variant of a Diffuser based planning method, by more than 2.0 points on average. Compared to Diffuser, DMEMM shows superior performance on all tasks, indicating that our method improves the consistency and optimality of diffusion model training in offline RL planning.

Overall, DMEMM achieves outstanding performance. With an average score of 87.9, DMEMM leads significantly, representing a substantial improvement over the second-highest average score of 84.6 achieved by HD-DA. These results clearly demonstrate the robustness and superiority of DMEMM in enhancing performance across various Gym tasks.

\begin{table*}[t]
\centering
\caption{This table presents the scores on D4RL locomotion suites for various comparison methods. Results are averaged over 5 seeds.}
\resizebox{\textwidth}{!}{
\Large
\renewcommand{\arraystretch}{1.15}
\begin{tabular}{lcccccccccc}
\Xhline{1pt}
\textbf{Gym Tasks} & \textbf{BC} & \textbf{CQL} & \textbf{IQL} & \textbf{DT} & \textbf{TT} & \textbf{MOReL} & \textbf{Diffuser} & \textbf{HDMI} & \textbf{HD-DA} & \textbf{DMEMM (Ours)} \\ \hline
Med-Expert HalfCheetah & 55.2 & 91.6 & 86.7 & 86.8 & 95.0 & 53.3 & 88.9$\pm$0.3 & 92.1$\pm$1.4 & 92.5$\pm$0.3 & \bf{94.6$\pm$1.2}\\ 
Med-Expert Hopper & 52.5 & 105.4 & 91.5 & 107.6 & 110.0 & 108.7 & 103.3$\pm$1.3 & 113.5$\pm$0.9 & 115.3$\pm$1.1 & \bf{115.9$\pm$1.6}\\ 
Med-Expert Walker2d & 107.5 & 108.8 & 109.6 & 108.1 & 101.9 & 95.6 & 106.9$\pm$0.2 & 107.9$\pm$1.2 & 107.1$\pm$0.1 & \bf{111.6$\pm$1.1}\\ 
\hline
Medium HalfCheetah & 42.6 & 44.0 & 47.4 & 42.6 & 46.9 & 42.1 & 42.8$\pm$0.3 & 48.0$\pm$0.9 & 46.7$\pm$0.2 & \bf{49.2$\pm$0.8} \\ 
Medium Hopper & 52.9 & 58.5 & 66.3 & 67.6 & 61.1 & 95.4 & 74.3$\pm$1.4 & 76.4$\pm$2.6 & 99.3$\pm$0.3 & \bf{101.2$\pm$1.4} \\ 
Medium Walker2d & 75.3 & 72.5 & 78.3 & 74.0 & 79.0 & 77.8 & 79.6$\pm$0.6 & 79.9$\pm$1.8 & 84.0$\pm$0.6 & \bf{86.5$\pm$1.5} \\ 
\hline
Med-Replay HalfCheetah & 36.6 & 45.5 & 44.2 & 36.6 & 41.9 & 40.2 & 37.7$\pm$0.5 & 44.9$\pm$2.0 & 38.1$\pm$0.7 & \bf{46.1$\pm$1.3}\\ 
Med-Replay Hopper & 18.1 & 95.0 & 94.7 & 82.7 & 91.5 & 93.6 & 93.6$\pm$0.4 & 99.6$\pm$1.5 & 94.7$\pm$0.7 & \bf{100.6$\pm$0.9} \\ 
Med-Replay Walker2d & 26.0 & 77.2 & 73.9 & 66.6 & 82.6 & 49.8 & 70.6$\pm$1.6 & 80.7$\pm$2.1 & 84.1$\pm$2.2 & \bf{85.8$\pm$2.6}\\ \hline
\textbf{Average} & 51.9 & 77.6 & 77.0 & 74.7 & 78.9 & 72.9 & 77.5 & 82.6 & 84.6 & \bf{87.9}\\ \hline
\end{tabular}
}
\label{tab:d4rl}
\end{table*}

\subsection{Experimental Results on Maze2D}
We present our experimental results on the Maze2D navigation tasks in Table \ref{tab:maze2d}, where the results are averaged over 5 seeds. The table shows that in both the Maze2D and Multi2D environments, particularly at the U-Maze and Medium difficulty levels, our proposed DMEMM method significantly outperforms other comparison methods. Specifically, on Maze2D tasks, DMEMM achieves a 4.0 point improvement over the state-of-the-art HD-DA method on the U-Maze task, and a 2.6 point increase on the Medium-sized maze. Compared to Diffuser, DMEMM shows an almost 20-point improvement. These results indicate that our method performs exceptionally well in generating planning solutions for navigation tasks.

However, HD-DA shows better performance on the large maze tasks. This is likely due to the hierarchical structure of HD-DA, which offers an advantage in larger, more complex environments by breaking long-horizon planning into smaller sub-tasks—an area where our method is not specifically designed to excel. Nevertheless, DMEMM remains competitive in larger environments, while demonstrating superior performance in smaller and medium-sized tasks.

\begin{table*}[t]
\centering
\caption{This table presents the scores on Maze2D navigation tasks for various comparison methods. Results are averaged over 5 seeds.}
\resizebox{0.8\textwidth}{!}{
\Large
\renewcommand{\arraystretch}{1.15}
\begin{tabular}{lcccccc}
\toprule
\textbf{Environment} & \textbf{MPPI} & \textbf{IQL} & \textbf{Diffuser}  & \textbf{HDMI} & \textbf{HD-DA} & \textbf{DMEMM (Ours)}\\ \midrule
Maze2D U-Maze & 33.2 & 47.4 & 113.9$\pm$3.1 & 120.1$\pm$2.5 & 128.4$\pm$3.6 & \bf{132.4$\pm$3.0}\\ 
Maze2D Medium & 10.2 & 34.9 & 121.5$\pm$2.7 & 121.8$\pm$1.6 & 135.6$\pm$3.0 & \bf{138.2$\pm$2.2}\\ 
Maze2D Large & 5.1 & 58.6 & 123.0$\pm$6.4 & 128.6$\pm$2.9 & \bf{155.8$\pm$2.5} & 153.2$\pm$3.3 \\ 
\hline
Multi2D U-Maze & 41.2 & 24.8 & 128.9$\pm$1.8 & 131.3$\pm$1.8 & 144.1$\pm$1.2 & \bf{145.6$\pm$2.6}\\ 
Multi2D Medium & 15.4 & 12.1 & 127.2$\pm$3.4 & 131.6$\pm$1.9 & 140.2$\pm$1.6 & \bf{140.8$\pm$2.2}\\ 
Multi2D Large & 8.0 & 13.9 & 132.1$\pm$5.8 & 135.4$\pm$2.5 & \bf{165.5$\pm$0.6} & 159.6$\pm$3.8\\ 
\bottomrule
\end{tabular}
}
\label{tab:maze2d}
\end{table*}

\subsection{Ablation Study}
We conduct an ablation study on our DMEMM method to evaluate the effectiveness of different components of our approach. We compare our full model with four different ablation variants: (1) DMEMM-w/o-weighting, which omits the weighting function of the reward-aware diffusion loss; (2) DMEMM-w/o-$\lambda_{\text{tr}}$, which omits the transition-based diffusion modulation loss; (3) DMEMM-w/o-$\lambda_{\text{rd}}$, which omits the reward-based diffusion modulation loss; and (4) DMEMM-w/o-tr-guide, which omits the transition guidance in the dual-guided sampling. The ablation study is conducted on the Hopper and Walker2D environments across all three levels of expert demonstration. The results of the ablation study are presented in Table \ref{tab:ablation}, which shows the scores on D4RL locomotion suites for all four ablation variants, averaged over 5 seeds.

The ablation study results highlight the importance of each component in the DMEMM method. Across both the Hopper and Walker2d environments, and at all three difficulty levels, the full DMEMM model achieves the best performance. Notably, both DMEMM-w/o-$\lambda_\text{tr}$ and DMEMM-w/o-tr-guide exhibit significant performance drops, emphasizing  the crucial role of incorporating transition dynamics in our method. The introduction of transition dynamics to the diffusion model greatly enhances the consistency and fidelity of the generated trajectory plans. Furthermore, DMEMM-w/o-$\lambda_\text{tr}$ and DMEMM-w/o-weighting show comparable performance, with the DMEMM-w/o-$\lambda_\text{tr}$ variant experiencing a slightly greater performance decrease. This suggests that our designed reward model plays a crucial role in improving the optimality of the generated trajectory plans.

Overall, the ablation study demonstrates that each component of our DMEMM method contributes significantly to its performance. Removing any of these components results in a noticeable decrease in performance, highlighting the importance of the weighting function, transition-based and reward-based diffusion modulation loss, and transition guidance in achieving optimal results in offline reinforcement learning tasks.

\begin{table*}[t]
\centering
\caption{This table presents the scores on D4RL locomotion suites for all four ablation variants. Results are averaged over 5 seeds.}
\resizebox{\textwidth}{!}{
\Large
\renewcommand{\arraystretch}{1.15}
\begin{tabular}{lccccc}
\Xhline{1pt}
\textbf{Gym Tasks} & \textbf{DMEMM} & \textbf{DMEMM-w/o-weighting} & \textbf{DMEMM-w/o-$\lambda_{tr}$} & \textbf{DMEMM-w/o-$\lambda_{rd}$} & \textbf{DMEMM-w/o-tr-guide} \\ \hline
Med-Expert Hopper & \bf{115.9$\pm$1.6} & 115.2$\pm$0.4 & 114.4$\pm$0.8 & 115.0$\pm$0.4 & 114.8$\pm$0.2 \\ 
Med-Expert Walker2d & \bf{111.6$\pm$1.1} & 110.4$\pm$0.8 & 108.4$\pm$1.2 & 110.4$\pm$0.6 & 109.9$\pm$1.0 \\ 
\hline
Medium Hopper & \bf{101.2$\pm$1.4} & 100.4$\pm$1.2 & 98.6$\pm$1.8 & 100.1$\pm$1.1 & 99.8$\pm$1.6 \\ 
Medium Walker2d & \bf{86.5$\pm$1.5} & 85.6$\pm$1.2 & 82.8$\pm$1.4 & 84.4$\pm$0.9 & 83.0$\pm$1.8 \\ 
\hline
Med-Replay Hopper & \bf{100.6$\pm$0.9} & 98.8$\pm$1.2 & 97.0$\pm$0.9 & 98.2$\pm$0.6 & 96.2$\pm$1.2 \\ 
Med-Replay Walker2d & \bf{85.8$\pm$2.6} & 84.6$\pm$2.2 & 82.2$\pm$1.7 & 83.7$\pm$2.5 & 82.6$\pm$3.2 \\
\hline
\end{tabular}
}
\label{tab:ablation}
\end{table*}

\subsection{Hyperparameter Sensitivity Analysis}
In this section, we analyze the sensitivity of the tradeoff parameters $\lambda_{\text{tr}}$ (transition-based diffusion modulation loss) and $\lambda_{\text{rd}}$ (reward-based diffusion modulation loss) to understand their impact on performance in offline RL tasks. The analysis is conducted on two environments: Hopper-Medium-Expert and Walker2D-Medium-Expert.

Figures \ref{fig:sensitivity} illustrate the performance sensitivity to the tradeoff parameters. For $\lambda_{\text{tr}}$, the performance peaks at approximately $\lambda_{\text{tr}} = 0.1$ in both the Walker2D-Medium-Expert and Hopper-Medium-Expert environments. Beyond this optimal point, performance declines notably, regardless of whether $\lambda_{\text{tr}}$ is increased or decreased. Similarly, for $\lambda_{\text{rd}}$, the performance also peaks around $\lambda_{\text{rd}} = 0.05$ in both environments. However, unlike $\lambda_{\text{tr}}$, performance shows little change when $\lambda_{\text{rd}}$ is adjusted within a small range, indicating that $\lambda_{\text{rd}}$ is less sensitive than $\lambda_{\text{tr}}$. Overall, the hyperparameter sensitivity analysis shows that both $\lambda_{\text{rd}}$ and $\lambda_{\text{tr}}$ have similar effects on performance and are robust across different tasks. Additionally, it confirms that the selected hyperparameters for our experiments are optimal.

\begin{figure*}[t]
\centering
    {\includegraphics[width=0.245\textwidth]{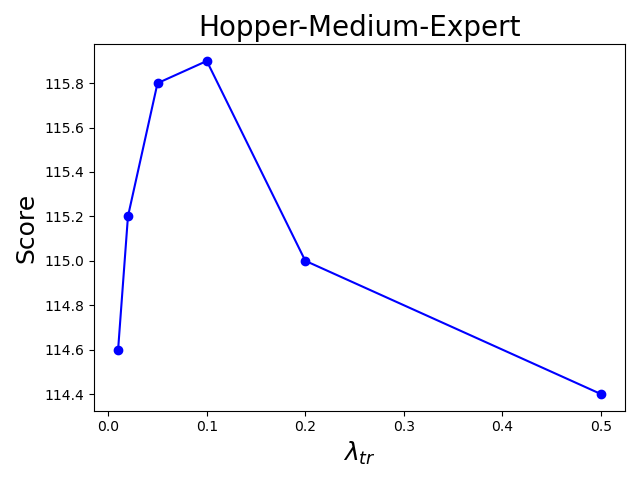}} 
	{\includegraphics[width=0.245\textwidth]{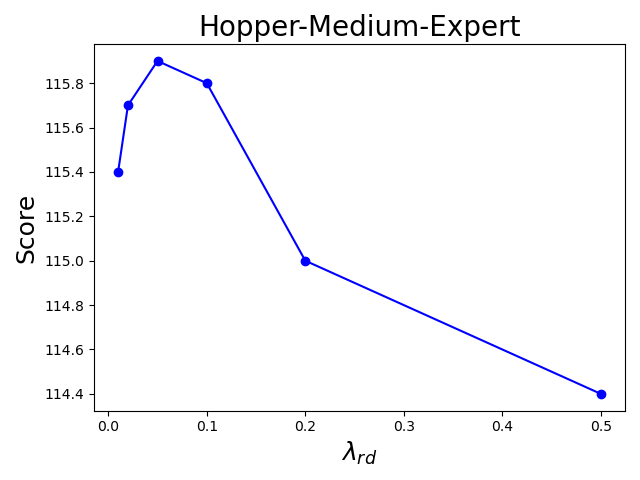}}
	{\includegraphics[width=0.245\textwidth]{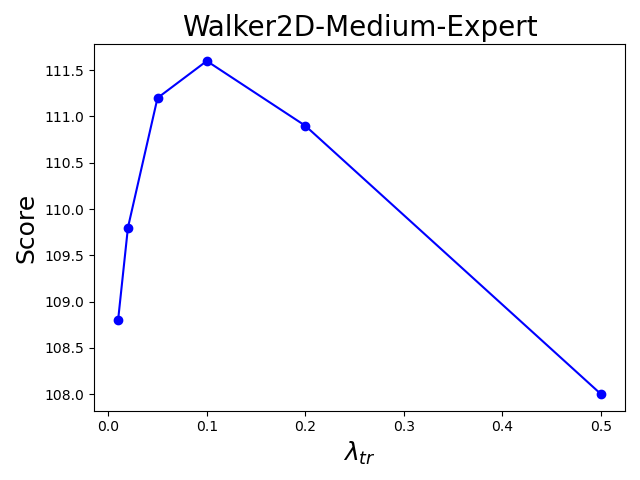}}
    {\includegraphics[width=0.245\textwidth]{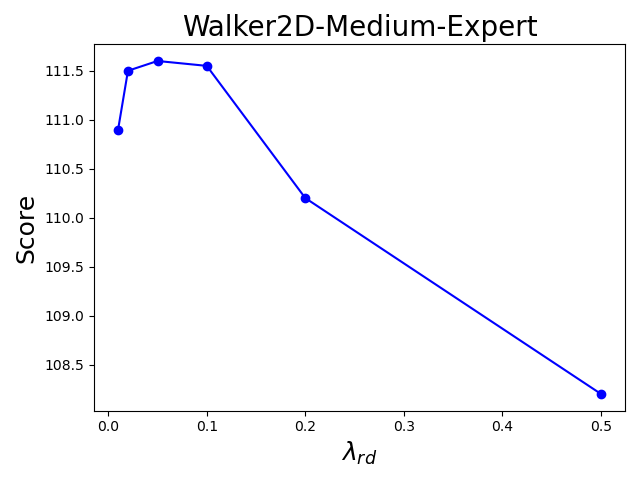}}
\vskip -.1in
\caption{
Hyperparameter sensitivity analysis of the tradeoff parameters for transition-based diffusion modulation loss ($\lambda_{tr}$) and reward-based diffusion modulation loss ($\lambda_{rd}$) on Hopper-Medium-Expert and Walker2D-Medium-Expert environments.
}
\label{fig:sensitivity}
\end{figure*}
%

\section{Conclusion}
In this work, we addressed a critical limitation of conventional diffusion-based planning methods in offline RL, which often overlook the consistency of transition dynamics in planned trajectories. To overcome this challenge, we proposed Diffusion Modulation via Environment Mechanism Modeling (DMEMM), a novel approach that integrates RL-specific environment mechanisms—particularly transition dynamics and reward functions—into the diffusion model training process. By modulating the diffusion loss with cumulative rewards and introducing auxiliary losses based on transition dynamics and reward functions, DMEMM enhances both the coherence and quality of the generated trajectories, ensuring they are plausible and optimized for policy learning. Our experimental results across multiple offline RL environments demonstrate the effectiveness of DMEMM, achieving state-of-the-art performance compared to previous diffusion-based planning methods. The proposed approach significantly improves the alignment of generated trajectories, addressing the discrepancies between offline data and real-world environments. This provides a promising framework for further exploration of diffusion models in RL and their potential practical applications.
\bibliography{iclr2025}

@article{janner2022planning,
  title={Planning with diffusion for flexible behavior synthesis},
  author={Janner, Michael and Du, Yilun and Tenenbaum, Joshua B and Levine, Sergey},
  journal={arXiv preprint arXiv:2205.09991},
  year={2022}
}

@article{li2023efficient,
  title={Efficient Planning with Latent Diffusion},
  author={Li, Wenhao},
  journal={arXiv preprint arXiv:2310.00311},
  year={2023}
}

@article{chen2024simple,
  title={Simple hierarchical planning with diffusion},
  author={Chen, Chang and Deng, Fei and Kawaguchi, Kenji and Gulcehre, Caglar and Ahn, Sungjin},
  journal={arXiv preprint arXiv:2401.02644},
  year={2024}
}

@inproceedings{ni2023metadiffuser,
  title={Metadiffuser: Diffusion model as conditional planner for offline meta-rl},
  author={Ni, Fei and Hao, Jianye and Mu, Yao and Yuan, Yifu and Zheng, Yan and Wang, Bin and Liang, Zhixuan},
  booktitle={International Conference on Machine Learning},
  pages={26087--26105},
  year={2023},
  organization={PMLR}
}

@inproceedings{fujimoto2019off,
  title={Off-policy deep reinforcement learning without exploration},
  author={Fujimoto, Scott and Meger, David and Precup, Doina},
  booktitle={International conference on machine learning},
  pages={2052--2062},
  year={2019},
  organization={PMLR}
}

@article{wu2019behavior,
  title={Behavior regularized offline reinforcement learning},
  author={Wu, Yifan and Tucker, George and Nachum, Ofir},
  journal={arXiv preprint arXiv:1911.11361},
  year={2019}
}

@article{kumar2020conservative,
  title={Conservative q-learning for offline reinforcement learning},
  author={Kumar, Aviral and Zhou, Aurick and Tucker, George and Levine, Sergey},
  journal={Advances in Neural Information Processing Systems},
  volume={33},
  pages={1179--1191},
  year={2020}
}

@article{kidambi2020morel,
  title={Morel: Model-based offline reinforcement learning},
  author={Kidambi, Rahul and Rajeswaran, Aravind and Netrapalli, Praneeth and Joachims, Thorsten},
  journal={Advances in neural information processing systems},
  volume={33},
  pages={21810--21823},
  year={2020}
}

@article{goyal2023robust,
  title={Robust markov decision processes: Beyond rectangularity},
  author={Goyal, Vineet and Grand-Clement, Julien},
  journal={Mathematics of Operations Research},
  volume={48},
  number={1},
  pages={203--226},
  year={2023},
  publisher={INFORMS}
}

@article{kostrikov2021offline,
  title={Offline reinforcement learning with implicit q-learning},
  author={Kostrikov, Ilya and Nair, Ashvin and Levine, Sergey},
  journal={arXiv preprint arXiv:2110.06169},
  year={2021}
}

@inproceedings{kumar2019stabilizing,
  title={Stabilizing off-policy q-learning via bootstrapping error reduction},
  author={Kumar, Aviral and Fu, Justin and Tucker, George and Levine, Sergey},
  booktitle={Advances in Neural Information Processing Systems},
  pages={11761--11771},
  year={2019}
}

@inproceedings{kostrikov2022offline,
  title={Offline reinforcement learning with implicit q-learning},
  author={Kostrikov, Ilya and Nachum, Ofir and Levine, Sergey and Tompson, Jonathan},
  booktitle={International Conference on Learning Representations},
  year={2022}
}

@article{chen2021decision,
  title={Decision transformer: Reinforcement learning via sequence modeling},
  author={Chen, Lili and Lu, Kevin and Rajeswaran, Aravind and Lee, Pieter Abbeel, Jason},
  journal={Advances in Neural Information Processing Systems},
  volume={34},
  pages={15084--15097},
  year={2021}
}

@inproceedings{janner2021trajectory,
  title={Trajectory transformer: Learning temporal dynamics for model-based planning},
  author={Janner, Michael and Li, Qiyang and Hsieh, Chang and Levine, Sergey and Finn, Chelsea},
  booktitle={Advances in Neural Information Processing Systems},
  pages={5792--5804},
  year={2021}
}

@inproceedings{emmons2022rvs,
  title={RvS: Reinforcement learning via supervised learning},
  author={Emmons, Scott and Lee, Honglak and Singh, Satinder},
  booktitle={Advances in Neural Information Processing Systems},
  year={2022}
}

@inproceedings{janner2022diffuser,
  title={Diffuser: Planning with Diffusion for Flexible Behavior Synthesis},
  author={Janner, Michael and Li, Qiyang and Cao, Xuesu and Finn, Chelsea},
  booktitle={Advances in Neural Information Processing Systems},
  year={2022}
}

@inproceedings{pdfd2022,
  title={PDFD: Planning with Diffusion via Flexible Dynamics},
  author={Author, First and Author, Second},
  booktitle={Proceedings of the Neural Information Processing Systems},
  year={2022}
}

@misc{d4rl,
  author = {Justin Fu and Aviral Kumar and Ofir Nachum and George Tucker and Sergey Levine},
  title = {D4RL: Datasets for Deep Data-Driven Reinforcement Learning},
  year = {2020},
  howpublished = {\url{https://github.com/rail-berkeley/d4rl}},
}

@InProceedings{sohl2015deep,
  title = {Deep Unsupervised Learning using Nonequilibrium Thermodynamics},
  author = {Sohl-Dickstein, Jascha and Weiss, Eric and Maheswaranathan, Niru and Ganguli, Surya},
  booktitle = {Proceedings of the 32nd International Conference on Machine Learning},
  pages = {2256--2265},
  year = {2015},
  editor = {Bach, Francis and Blei, David},
  volume = {37},
  series = {Proceedings of Machine Learning Research},
  address = {Lille, France},
  month = {07--09 Jul},
  publisher = {PMLR},
  url = {https://proceedings.mlr.press/v37/sohl-dickstein15.html}
}

@inproceedings{ho2020denoising,
  title={Denoising Diffusion Probabilistic Models},
  author={Ho, Jonathan and Jain, Ajay and Abbeel, Pieter},
  booktitle={Advances in Neural Information Processing Systems},
  pages={6840--6851},
  year={2020},
  organization={Curran Associates, Inc.},
  url={https://proceedings.neurips.cc/paper/2020/hash/4c5bcfec45690f3b1284a8e0c60242a0-Abstract.html}
}

@book{sutton2018reinforcement,
  title={Reinforcement learning: An introduction},
  author={Sutton, Richard S and Barto, Andrew G},
  year={2018},
  publisher={MIT press}
}

@misc{dhariwal2021diffusionmodelsbeatgans,
      title={Diffusion Models Beat GANs on Image Synthesis}, 
      author={Prafulla Dhariwal and Alex Nichol},
      year={2021},
      eprint={2105.05233},
      archivePrefix={arXiv},
      primaryClass={cs.LG},
      url={https://arxiv.org/abs/2105.05233}, 
}

@inproceedings{
  song2021scorebased,
  title={Score-Based Generative Modeling through Stochastic Differential Equations},
  author={Yang Song and Jascha Sohl-Dickstein and Diederik P Kingma and Abhishek Kumar and Stefano Ermon and Ben Poole},
  booktitle={International Conference on Learning Representations},
  year={2021},
  url={https://openreview.net/forum?id=PxTIG12RRHS}
}

@misc{nichol2021improveddenoisingdiffusionprobabilistic,
      title={Improved Denoising Diffusion Probabilistic Models}, 
      author={Alex Nichol and Prafulla Dhariwal},
      year={2021},
      eprint={2102.09672},
      archivePrefix={arXiv},
      primaryClass={cs.LG},
      url={https://arxiv.org/abs/2102.09672}, 
}

@misc{ho2022classifierfreediffusionguidance,
      title={Classifier-Free Diffusion Guidance}, 
      author={Jonathan Ho and Tim Salimans},
      year={2022},
      eprint={2207.12598},
      archivePrefix={arXiv},
      primaryClass={cs.LG},
      url={https://arxiv.org/abs/2207.12598}, 
}

@misc{yu2020mopomodelbasedofflinepolicy,
      title={MOPO: Model-based Offline Policy Optimization}, 
      author={Tianhe Yu and Garrett Thomas and Lantao Yu and Stefano Ermon and James Zou and Sergey Levine and Chelsea Finn and Tengyu Ma},
      year={2020},
      eprint={2005.13239},
      archivePrefix={arXiv},
      primaryClass={cs.LG},
      url={https://arxiv.org/abs/2005.13239}, 
}

@misc{janner2021offlinereinforcementlearningbig,
      title={Offline Reinforcement Learning as One Big Sequence Modeling Problem}, 
      author={Michael Janner and Qiyang Li and Sergey Levine},
      year={2021},
      eprint={2106.02039},
      archivePrefix={arXiv},
      primaryClass={cs.LG},
      url={https://arxiv.org/abs/2106.02039}, 
}

@misc{levine2020offlinereinforcementlearningtutorial,
      title={Offline Reinforcement Learning: Tutorial, Review, and Perspectives on Open Problems}, 
      author={Sergey Levine and Aviral Kumar and George Tucker and Justin Fu},
      year={2020},
      eprint={2005.01643},
      archivePrefix={arXiv},
      primaryClass={cs.LG},
      url={https://arxiv.org/abs/2005.01643}, 
}

@inproceedings{kingma2014auto,
  title={Auto-encoding variational Bayes},
  author={Kingma, Diederik P and Welling, Max},
  booktitle={International Conference on Learning Representations (ICLR)},
  year={2014},
  url={https://arxiv.org/abs/1312.6114}
}
\bibliographystyle{iclr2025}
\appendix
\section{Diffusion Training Algorithm}
\begin{algorithm}[t]
\caption{Diffusion Training}
\label{alg:training}
\begin{algorithmic}
\REQUIRE Offline data $\Dc=\{(s_0^i, a_0^i, r_0^i, s_1^i, a_1^i, r_1^i, \dots, s_T^i, a_T^i, r_T^i)\}$.
\STATE Learn transition model 
$\widehat{\Tc}(s_t, a_t)$ and reward function $\widehat{\Rc}(s_t, a_t)$ from offline data $\Dc$.
\STATE Initialize noise network $\epsilon_\theta(\tau^k,k)$.
\WHILE{not converged}
    \STATE Sample a trajectory from offline data $\Btau^0 \sim \Dc$.
    \STATE Sample a random diffusion step $k \sim \Uc(1,K)$.
    \STATE Sample a random noise $\Bepsilon \sim \Nc(0, \bI)$.
    \STATE Calculate the gradient $\nabla_\theta L_{\text{total}}$ of Eq. (\ref{eqa:total_loss}) and take gradient descent step.
\ENDWHILE
\end{algorithmic}
\end{algorithm}
The complete training process of the diffusion model is presented in Algorithm \ref{alg:training}. Prior to training the diffusion model, 
a probabilistic transition
model $\widehat{\Tc}(s_t, a_t)$ and a reward model $\widehat{\Rc}(s_t, a_t)$ are learned from the offline dataset $\Dc$. Afterward, the noise network is initialized and iteratively trained. During each iteration, an original trajectory $\Btau^0$ is sampled from the offline dataset $\Dc$, along with a randomly selected diffusion step $k$ and noise sample $\Bepsilon$. Gradient descent is then applied to minimize the total loss 
$L_\text{total}$. 

\section{Proof of Proposition \ref{prop1}}
In this section, we present the proof of Proposition \ref{prop1}.
\begin{proof}
To incorporate key RL mechanisms into the training of the diffusion model, we explore the denoising process and trace the denoised data through the reverse diffusion process. Let $\hBtau^0$
represent the denoised output trajectory.
It can be gradually denoised using the reverse process, following the chain rule: $\hBtau^0\sim p_\theta(\Btau^K)\prod_{k=1}^{K} p_\theta(\Btau^{k-1}|\Btau^k)$,
where the detailed reverse process is defined in Eq. (\ref{eq:reversediff}) and Eq. (\ref{eqa:mu}). Starting from an intermediate trajectory $\Btau^k$ at step $k$,
by combining these two equations, 
the trajectory at the next diffusion step, $k-1$, can be directly sampled from the distribution:
\begin{equation}
\hBtau^{k-1}\sim \Nc\left(\frac{1}{\sqrt{\alpha_k}} \left( \Btau^k - \frac{1 - \alpha_k}{\sqrt{1 - \bar{\alpha}_k}} \epsilon_\theta (\Btau^k, k) \right),\sigma_k^2\bI\right).
\end{equation}
By applying the reparameterization trick \citep{kingma2014auto}, we can derive a closed-form solution for the above distribution. Let $\Bepsilon_{k}$ represent the noise introduced in the reverse process $p_\theta(\Btau^{k-1}|\Btau_k)$, and the denoised trajectory can then be formulated as:
\begin{equation}
\begin{aligned}
\hBtau^{k-1}
&= \frac{1}{\sqrt{\alpha_k}} \left( \Btau^k - \frac{1 - \alpha_k}{\sqrt{1 - \bar{\alpha}_k}} \epsilon_\theta (\Btau^k, k) \right) + \sigma_k \Bepsilon_k \\
&= \frac{1}{\sqrt{\alpha_k}} \Btau^k - \frac{1 - \alpha_k}{\sqrt{(1 - \bar{\alpha}_k)\alpha_k}} \epsilon_\theta (\Btau^k, k) +\sigma_k \Bepsilon_k.
\end{aligned}
\end{equation}
In the following diffusion step $k-2$, the denoised data $\hBtau^{k-2}$ is sampled from a similar Gaussian distribution. By the Central Limit Theorem, $\hBtau^{k-1}$ serves as an unbiased estimate of $\Btau^{k-1}$. Therefore, the denoised data $\hBtau^{k-2}$ can be expressed as follows:
\begin{equation}
\begin{aligned}
\hBtau^{k-2}
&\sim \Nc\left(\frac{1}{\sqrt{\alpha_{k-1}}} \left( \Btau^{k-1} - \frac{1 - \alpha_{k-1}}{\sqrt{1 - \bar{\alpha}_{k-1}}} \epsilon_\theta (\Btau^{k-1}, k-1) \right),\sigma_{k-1}^2\bI\right) \\
&= \frac{1}{\sqrt{\alpha_{k-1}}} \left( \hBtau^{k-1} - \frac{1 - \alpha_{k-1}}{\sqrt{1 - \bar{\alpha}_{k-1}}} \epsilon_\theta (\Btau^{k-1}, k-1) \right) + \sigma_{k-1} \Bepsilon_{k-1} \\
&= \frac{1}{\sqrt{\alpha_{k-1}}} \left( \frac{1}{\sqrt{\alpha_k}} \Btau^k - \frac{1 - \alpha_k}{\sqrt{(1 - \bar{\alpha}_k)\alpha_k}} \epsilon_\theta (\Btau^k, k) - \frac{1 - \alpha_{k-1}}{\sqrt{1 - \bar{\alpha}_{k-1}}} \epsilon_\theta (\Btau^{k-1}, k-1) +\sigma_k \Bepsilon_k \right) \\
&\quad +  \sigma_{k-1} \Bepsilon_{k-1} \\
&= \frac{1}{\sqrt{\alpha_k\alpha_{k-1}}} \Btau^k 
- \frac{1 - \alpha_k}{\sqrt{(1 - \bar{\alpha}_k)\alpha_k\alpha_{k-1}}} \epsilon_\theta (\Btau^k, k) 
- \frac{1 - \alpha_{k-1}}{\sqrt{(1 - \bar{\alpha}_{k-1})\alpha_{k-1}}} \epsilon_\theta (\Btau^{k-1}, k-1) \\
&\quad + \frac{1}{\sqrt{\alpha_{k-1}}} \sigma_k \Bepsilon_k + \sigma_{k-1} \Bepsilon_{k-1}.
\end{aligned}
\end{equation}
The introduced noise $\Bepsilon_{k-1}$ in diffusion step $k-1$  can be combined with the noise $\Bepsilon_k$ at diffusion step $k$ into a joint noise term, $\bar{\Bepsilon}_{k-1}$, by merging two Gaussian distributions, $\Nc(0, \frac{\sigma_k^2}{\alpha_{k-1}}\bI)$ and $\Nc(0, \sigma_{k-1}^2\bI)$, into $\Nc(0, (\frac{\sigma_k^2}{\alpha_{k-1}}+\sigma_{k-1}^2)\bI)$. Consequently, we obtain the distribution for the denoised data $\hBtau^{k-2}$ with only directly computable terms, where
\begin{equation}
\begin{aligned}
\hBtau^{k-2}
&= \frac{1}{\sqrt{\alpha_k\alpha_{k-1}}} \Btau^k 
- \frac{1 - \alpha_k}{\sqrt{(1 - \bar{\alpha}_k)\alpha_k\alpha_{k-1}}} \epsilon_\theta (\Btau^k, k) 
- \frac{1 - \alpha_{k-1}}{\sqrt{(1 - \bar{\alpha}_{k-1})\alpha_{k-1}}} \epsilon_\theta (\Btau^{k-1}, k-1) \\
&\quad + \sqrt{\frac{\sigma_k^2}{\alpha_{k-1}} + \sigma_{k-1}^2} \bar{\Bepsilon}_{k-1} \\
&\sim \Nc\left(
\frac{1}{\sqrt{\alpha_k\alpha_{k-1}}} \Btau^k 
- \frac{1 - \alpha_k}{\sqrt{(1 - \bar{\alpha}_k)\alpha_k\alpha_{k-1}}} \epsilon_\theta (\Btau^k, k) 
- \frac{1 - \alpha_{k-1}}{\sqrt{(1 - \bar{\alpha}_{k-1})\alpha_{k-1}}} \epsilon_\theta (\Btau^{k-1}, k-1), \right. \\
&\quad \left. \left( \frac{\sigma_k^2}{\alpha_{k-1}} + \sigma_{k-1}^2 \right)\bI \right).
\end{aligned}
\end{equation}
By repeating the denoising process for $k$ iterations, we can ultimately obtain a closed-form representation of the denoised data $\hBtau^0$.
\begin{equation}
\begin{aligned}
\hBtau^0
&=\frac{1}{\sqrt{\prod_{i=1}^k\alpha_i}}\Btau^k-\sum_{i=1}^k \frac{1-\alpha_i}{\sqrt{(1-\bar{\alpha}_i)\prod_{j=1}^i \alpha_j}}\epsilon_\theta(\Btau^i,i) + \sqrt{\sigma_1^2+\sum_{i=2}^k\frac{\sigma_i^2}{\prod_{j=1}^{i-1}\alpha_j}}\bar{\epsilon}_1\\
&=\frac{1}{\sqrt{\bar{\alpha}_k}}\Btau^k-\sum_{i=1}^k \frac{1-\alpha_i}{\sqrt{(1-\bar{\alpha}_i)\bar{\alpha}_i}}\epsilon_\theta(\Btau^i,i)+\sqrt{\sigma_1^2+\sum_{i=2}^k\frac{\sigma_i^2}{\bar{\alpha}_{i-1}}}\bar{\epsilon}_1.
\end{aligned}
\end{equation}
Using the closed-form representation of the reparameterization trick, the final denoised data $\hBtau^0$ follows a Gaussian distribution, expressed as $\hBtau^0 \sim \Nc(\hmu_\theta(\Btau^k, k), \widehat{\sigma}^2\mathbf{I})$. The mean $\hmu_\theta(\Btau^k, k)$ captures the denoising trajectory and is formulated as:
\begin{equation}
\hmu_\theta(\Btau^k, k)=\frac{1}{\sqrt{\bar{\alpha}_k}}\Btau^k-\sum_{i=1}^k \frac{1-\alpha_i}{\sqrt{(1-\bar{\alpha}_i)\bar{\alpha}_i}}\epsilon_\theta(\Btau^i,i).
\end{equation}
Similarly, the covariance $\widehat{\sigma}^2$ accounts for the accumulation of noise over all diffusion steps and is written as:
\begin{equation}
\widehat{\sigma}^2=\sigma_1^2+\sum_{i=2}^k\frac{\sigma_i^2}{\bar{\alpha}_{i-1}}.
\end{equation}
\end{proof}

\end{document}